\documentclass[11pt]{article}

\usepackage{hyperref}
\usepackage{geometry}        
\usepackage{times}           
\usepackage{float}           
\usepackage{subcaption}
\usepackage{color}

\usepackage{natbib}
\setcitestyle{round}

\usepackage[utf8]{inputenc} 
\usepackage[T1]{fontenc}    
\usepackage{url}            
\usepackage{booktabs}       
\usepackage{amsfonts}       
\usepackage{nicefrac}       
\usepackage{microtype}      
\usepackage{xcolor}         
\usepackage{graphicx}
\usepackage{amsmath}
\usepackage{amsthm}
\usepackage{amssymb}
\usepackage{enumitem}
\usepackage{algorithm}
\usepackage{algorithmic}

\usepackage{pgfplots}
\usepackage{pgfplotstable}
\usepgfplotslibrary{groupplots}
\usepackage{caption}
\pgfplotsset{compat=1.18}
\usepgfplotslibrary{fillbetween}
\usepackage{tikz}
\usepackage{multirow}

\newtheorem{definition}{Definition}
\newtheorem{remark}{Remark}
\newtheorem{proposition}{Proposition}
\newtheorem{assumption}{Assumption}
\newtheorem{lemma}{Lemma}

\definecolor{myredfig1}{RGB}{220,20,60}
\definecolor{myred}{RGB}{255,31,91}
\definecolor{mygreen}{RGB}{0,205,108}
\definecolor{myblue}{RGB}{0,154,222}
\definecolor{mypurple}{RGB}{175,88,186}
\definecolor{myyellow}{RGB}{255,198,30}
\definecolor{myorange}{RGB}{242,133,34}
\definecolor{mygray}{RGB}{160,177,186}
\definecolor{mybrown}{RGB}{166,118,29}

\newcommand{\smalleq}[2][0.95]{\scalebox{#1}{$#2$}}

\title{\textbf{A Temporal Difference Method} \\ \textbf{for Stochastic Continuous Dynamics}}
\author{
  \textbf{Haruki Settai} \\
  University of Tokyo \\
  \texttt{sharuk@g.ecc.u-tokyo.ac.jp}
  \and
  \textbf{Naoya Takeishi} \\
  \normalsize University of Tokyo \\
  \normalsize \texttt{ntake@g.ecc.u-tokyo.ac.jp}
  \and
  \textbf{Takehisa Yairi} \\
  \normalsize University of Tokyo \\
  \normalsize \texttt{yairi@g.ecc.u-tokyo.ac.jp}
}
\date{}

\begin{document}

\maketitle

\begin{abstract}
  For continuous systems modeled by dynamical equations such as ODEs and SDEs, Bellman's Principle of Optimality takes the form of the Hamilton-Jacobi-Bellman (HJB) equation, which provides the theoretical target of reinforcement learning (RL). Although recent advances in RL successfully leverage this formulation, the existing methods typically assume the underlying dynamics are known a priori because they need explicit access to the coefficient functions of dynamical equations to update the value function following the HJB equation. We address this inherent limitation of HJB-based RL; we propose a model-free approach still targeting the HJB equation and propose the corresponding temporal difference method. We establish exponential convergence of the idealized continuous-time dynamics and empirically demonstrate its potential advantages over transition–kernel–based formulations. The proposed formulation paves the way toward bridging stochastic control and model-free reinforcement learning.
\end{abstract}

\section{Introduction}
Reinforcement learning (RL) has been successfully applied in various domains, ranging from discrete systems like board games \citep{shs:18} to systems that are continuous in both state and time such as robotic control \citep{kbp:13}. RL research has been advancing through a variety of approaches, and much of the work has focused on improving methods by refining objective functions \citep{slm:15,swd:17}, balancing exploration and exploitation \citep{hza:18}, and developing more effective architectures \citep{hlb:19}. These efforts have significantly advanced the field, leading to the development of various successful algorithms.

In contrast to the studies primarily targeting algorithmic components or optimization techniques, we focus on the continuity of time and explore what we call continuous RL, where the dynamics of systems are described by ordinary differential equations (ODEs) or stochastic differential equations (SDEs), which is a relatively underexplored aspect of RL. Although many RL methods have been applied, sometimes heuristically, to both discrete and continuous systems, the continuity of the system has not necessarily been fully exploited, even when it is known in advance. Laying a foundation for utilizing prior knowledge of continuity is important toward more effective learning and decision-making methods, particularly for practical applications such as robot control and autonomous driving, which typically fall into the target of continuous RL.

A way to incorporate prior knowledge of the system's continuity into the learning process is to use the Hamilton-Jacobi-Bellman (HJB) equation. In the Bellman equation, continuity is encoded in the transition kernel. However, in model-free RL, where transitions are approximated using samples, this continuity information is lost because the transition kernel is not explicitly modeled. On the other hand, the HJB equation retains the continuity information in the argument of the expectation rather than in the transition kernel. This property allows the HJB equation to keep taking continuity into account even under sample-based approximations. However, because the HJB equation depends on the coefficient functions of the system's dynamics, prior work has largely been limited to model-based approaches \citep{mb:97, yhl:21}.

In this paper, we introduce a temporal difference (TD) method based on the HJB equation, namely \emph{differential TD} (dTD), achieved through sample-based approximation of the expectation term in the HJB equation. dTD enables policy evaluation without requiring knowledge or estimation of the system dynamics, while incorporating the continuity of the dynamics into the learning process. It is compatible with on-policy methods such as A2C \citep{mbm:16} and PPO \citep{swd:17}, and we demonstrate its effectiveness on Mujoco \citep{tet:12} tasks including Hopper, HalfCheetah, Ant, and Humanoid. The codes for the proposed method are available at \url{https://github.com/4thhia/differential_TD}.

\section{Related Work}
\paragraph{Deterministic Dynamics}
The study of continuous RL for ODE systems can be traced back to studies such as \citet{b:94, m:97, d:00, m:06}. \citet{b:94} discovered that the Q-function collapses in continuous RL, which was rigorously proven and extended to deep RL in \citet{tbo:19}. In \citet{m:97}, model-free approaches for ODE systems were first studied. \citet{d:00} was the first to introduce TD for ODE systems and extended it to TD($\lambda$) and actor-critic. \citet{m:06} investigated the estimation of policy gradients for ODE systems and proposed a pathwise derivative approach. More recently, \citet{vc:17} developed a model-free RL framework for deterministic linear systems. \citet{ksy:21} proposed a model-free Q-learning approach in which the control is derived from the HJB equation, while the learning target is based on the conventional Bellman equation. \citet{yhl:21} introduced a model-based method that leverages the Neural ODE framework to enable continuous-time optimization using learned system dynamics.

\paragraph{Stochastic Dynamics}
One of the earliest works on RL in SDE systems is \citet{mb:97}, which takes a model-based approach. However, research in this direction remained largely unexplored until recently. In the past few years, a growing body of work has emerged that aims to establish theoretical foundations for RL in stochastic dynamics.
\citet{wzz:20} introduced an entropy-regularized relaxed control formulation and provided a comprehensive analysis in the LQR setting. \citet{tzz:22} further demonstrated the well-posedness of the HJB equation within this relaxed control framework. \citet{jz:22} showed that Bellman optimality is equivalent to maintaining the martingale property of a suitably defined stochastic process, and proposed a corresponding algorithm. However, their method requires access to full trajectory information and thus applies only to finite-horizon settings. Building on this approach, \citet{jz2:22, jz:23} proposed actor-critic and Q-learning algorithms for finite-horizon SDE systems, respectively. \citet{zty:23} extended key theoretical tools such as the state visitation distribution and the performance difference lemma to the continuous-time setting and applied them to TRPO and PPO. Independently, \citet{kb:23} proposed a variance-reduction technique for policy evaluation.

\section{Background}

\subsection{Problem Setting}
We consider a continuous RL setting where the state space is $\mathcal S\subset \mathbb R^n$ and the action space is $\mathcal A$. We suppose that the dynamics of the state are governed by the following controlled SDE:
\begin{equation}
dS_t = \mu(S_t, A_t)dt + \sigma(S_t, A_t)dB_t, \label{eq:sde}
\end{equation}
where $\mu: \mathcal{S} \times \mathcal{A} \rightarrow \mathbb R^{n}$, $\sigma:\mathcal{S} \times \mathcal{A}\rightarrow \mathbb R^{n\times m}$, and $(B_t)_{t\geq 0}$ is the $m$-dimensional Brownian motion. Note that the state evolution is influenced by both the inherent noise in the system as well as the randomness induced by the stochastic policy $\pi: \mathcal S\rightarrow \mathcal P(\mathcal A)$, where $\mathcal P(\mathcal A)$ is the space of probability distribution over the action space. Thus, the expectation related to this SDE is expressed as $\mathbb E_{p_{\pi}}[\cdot]$, where $p_{\pi}$ denotes the transition kernel corresponding to this SDE. For simplicity, we assume that the stochastic process \eqref{eq:sde} is well-defined; see Appendix~\ref{app:justification} for a detailed justification. We here focus on SDE systems because we can recover results for ODE systems in the limit of $\sigma=0$.

We primarily focus on the model-free setting, where the agent has no prior knowledge of the dynamics coefficients $\mu$ and $\sigma$ in Eq.~\eqref{eq:sde}. Throughout this paper, the term RL inherently refers to this model-free setting unless specifically stated otherwise (e.g., as model-based RL).

\subsection{Continuous RL}
\label{sec:continuous_RL}
The dynamics governed by the SDE in Eq.~\eqref{eq:sde} exhibit the Markov property under a Markov policy. (Informally, the infinitesimal evolution depends only on the current state–action pair, as the coefficients of the SDE depend on $(S_t,A_t)$ at time $t$.) This confirms that continuous RL falls within the general framework of an MDP. For example, The Bellman optimality equation can be established as follows:
\begin{equation*}
V^{*}(s_t) = \underset{\pi}{\max}\ \mathbb E_{p_{\pi}}\left[\int_{t}^{t'}\text{e}^{-\gamma (\tau-t)}\rho(s_\tau, A_\tau)d\tau + \text{e}^{-\gamma (t'-t)} V^{*}(S_{t'})\right],
\end{equation*}
where $\rho: \mathcal{S} \times \mathcal{A}\rightarrow \mathbb R$ is the reward rate function, and $\gamma\in (0, \infty)$ is the constant discount factor. Here, $V^{*}(s)$ is the optimal value function, defined as:
\begin{equation*}
V^{*}(s) := \underset{\pi}{\max}\ \mathbb{E}_{p_{\pi}}\left[\int_{t}^{\infty}\text{e}^{-\gamma(\tau-t)}\rho(S_\tau, A_\tau) d\tau\ \middle|\ S_{t} = s\right].
\end{equation*}

When learning based on discrete observations of a continuous system, such as in simulations, one can discretize the problem as follows:
\begin{equation}
V^{*}(s_t) = \underset{\pi}{\max}\ \mathbb E_{p_{\pi}}\left[\rho(s_t, A_t)\Delta t + \text{e}^{-\gamma \Delta t} V^{*}(S_{t+\Delta t})\right], \label{eq:cbe}
\end{equation}
where $\Delta t$ is a small time interval and need not be constant.

This example shows that standard discrete RL methods can be straightforwardly applied to continuous RL (though the use of Q-functions in continuous time requires careful consideration \citep{tbo:19}). However, this approach may not be sufficiently efficient when employing model-free methods that approximate the expectation using samples. This inefficiency arises because the system's underlying SDE dynamics are only encoded in the transition kernel (i.e., the subscript of the expectation). Dropping this kernel via sample-based approximation thus disregards the prior knowledge of continuity and fails to leverage this structural information. This implies that the agent may not even be aware of whether it is operating in a discrete or a continuous system.

Since we are restricting the problem class from general MDPs to continuous RL, it is desirable to develop more efficient algorithms. Specifically, we desire algorithms where the agent can at least perceive and leverage the system's continuity information.

\subsection{Natural Target of Continuous RL}
We focus on TD learning, a fundamental approach in RL. TD learning naturally target two fundamental types of Bellman equations: the optimality equation and the expectation equation. This classification, combined with the V- or Q-function, leads to four primary candidates for the target.

Among these, the V-Bellman optimality equation is immediately excluded as it is not compatible with model-free framework. Furthermore, we will show in Section~\ref{sec:main} that the Q-Bellman optimality equation, while model-free in discrete RL, cannot be effectively learned by TD in the context of continuous RL due to the difficulty of the maximization operator. Thus, in continuous RL, the natural scope is policy evaluation.

The choice is now between the V- and Q-Bellman expectation equations. Although the Q-Bellman expectation equation is compatible with our proposed methodology, Q-functions often degenerate in continuous-time systems~\citep{tbo:19}, introducing unnecessary complexity. Therefore, to maintain clarity regarding the core principles, we focus on policy evaluation using the V-Bellman expectation equation in this paper. This V-Bellman expectation equation is expressed as follows in continuous RL.

\begin{equation}
V^\pi(s_t) = \mathbb E_{p_{\pi}}\left[\rho(s_t, A_t)\Delta t + \text{e}^{-\gamma \Delta t} V^{\pi}(S_{t+\Delta t})\right], \label{eq:cbe2}
\end{equation}

\subsubsection{Contrast with Stochastic Control}
Much of the prior work on continuou RL originates from the field of stochastic control. These approaches often start by directly introducing the HJB equation, which is the continuous counterpart of the Bellman optimality equation~\eqref{eq:cbe}. For instance, the HJB equation can be expressed as:
(see Appendix~\ref{app:hjb2} for more detail):
\begin{equation}
\smalleq{\begin{aligned}
V^{*}(s_t) = \frac{1}{\gamma} \max_{\pi}\ 
\mathbb{E}_{\pi} \Bigg[
    \rho(s_t, A_t)
    &+ \sum_{i=1}^{n} \mu^i(s_t, A_t) 
    \frac{\partial V^{*}(s)}{\partial s^i} \bigg|_{s=s_t} \\
    &+ \frac{1}{2} \sum_{i=1}^{n} \sum_{j=1}^{n}
    [\sigma(s_t, A_t)\sigma^\top(s_t, A_t)]^{ij}
    \frac{\partial^2 V^{*}(s)}{\partial s^i \partial s^j} 
    \bigg|_{s=s_t} 
\Bigg],
\end{aligned}}
\label{eq:hjb}
\end{equation}
where $\mu^i$ and $[\sigma\sigma^\top]^{ij}$ denote the $i$-th element of $\mu$ and the $(i, j)$-th element of $\sigma\sigma^\top$, respectively.

\noindent However, in the context of continuous RL, the introduction of HJB equation immediately restricts the methodology to model-based methods, even if we replace the $V$ function with a $Q$ function, as discussed in Section~\ref{sec:main}. Thus, adopting the HJB equation offers no benefit over the standard Bellman equation~\eqref{eq:cbe} for TD learning.

Independently of this stochastic-control line of work, our motivation—to design algorithms that allow the agent to leverage the system's underlying continuity in a model-free manner—naturally leads us to introduce a variant of the HJB equation. And this variant, as noted earlier, is built upon the V-Bellman expectation equation~\eqref{eq:cbe2}.

\section{TD Method for Stochastic Continuous Dynamics}
\label{sec:main}

A natural way to inform the agent that the system follows an SDE is to embed the model directly into the update rule. As discussed in Section~\ref{sec:continuous_RL}, with the standard Bellman expectation equation~\eqref{eq:cbe2} this information is lost when we pass to a sample-based approximation, which drops the expectation subscript $p_\pi$ where the model is embedded. Therefore, by embedding the SDE information in the argument rather than in the subscript of the expectation, one can prevent this information from being lost even under sample approximation. This can be achieved by further transforming the Bellman expectation equation using the SDE, i.e., by expanding $V(S_{t+\Delta t})$ via Itô's formula, resulting in a variant of the HJB equation (the derivation is identical to that of the HJB equation; see Appendix~\ref{app:hjb2} for more details):
\begin{equation}
\smalleq{\begin{aligned}
V^{\pi}(s_t) = \frac{1}{\gamma} \mathbb{E}_{\pi} \Bigg[
    \rho(s_t, A_t)
    &+ \sum_{i=1}^{n} \mu^i(s_t, A_t)
    \frac{\partial V^{\pi}(s)}{\partial s^i} \bigg|_{s_t}
    \!\!
    + \frac{1}{2} \sum_{i=1}^{n} \sum_{j=1}^{n}
    [\sigma(s_t, A_t)\sigma^\top(s_t, A_t)]^{ij}
    \frac{\partial^2 V^{\pi}(s)}{\partial s^i \partial s^j}
    \bigg|_{s_t}
\Bigg].
\end{aligned}}
\label{eq:ehjb}
\end{equation}
We call this equation the HJB for a fixed policy.

Now, are we all ready to implement model-free value iteration just by approximating the expectation on the right-hand side of \eqref{eq:ehjb} with samples? The answer is no because the argument of the expectation includes the coefficient functions of the SDE, $\mu$ and $\sigma$, making it impossible to directly approximate the expectation with samples.

\subsection{Deriving TD from the HJB equation}
We now present our main theoretical result. It gives the foundation for our model-free formulation of temporal-difference learning based on the HJB equation. The idea is that the drift and diffusion terms in the HJB equation can be equivalently expressed using limits of sample-based finite differences.

\begin{proposition}
When a stochastic process $(S_t)_{t \geq 0}$ follows the SDE in \eqref{eq:sde}, we have
\begin{equation}
\mathbb E_{\pi}\left[\mu^i(s_t, A_t)\right] = \underset{\Delta t\rightarrow 0}{\lim}\mathbb E_{p_{\pi}}\left[\frac{S_{t+\Delta t}^i-s_t^i}{\Delta t}\right] \label{eq:claim1}
\end{equation}
and
\begin{equation}
\mathbb E_{\pi}\left[[\sigma(s_t, A_t)\sigma^\top(s_t, A_t)]^{ij}\right] = \underset{\Delta t\rightarrow 0}{\lim}\mathbb E_{p_{\pi}}\left[\frac{(S_{t+\Delta t}^i - s_{t}^i)(S_{t+\Delta t}^j - s_{t}^j)}{\Delta t}\right]. \label{eq:claim2}
\end{equation}
\label{thm:coef}
\end{proposition}
\begin{proof}
For the first claim, we expand the $i$-th component of the SDE~\eqref{eq:sde} using the Ito formula:
\begin{align}
S_{t+\Delta t}^{i} =& s_t^i + \mu^i(s_t, A_t)\Delta t + \underset{j=1}{\overset{m}{\sum}}\sigma^{ij}(s_t, A_t)(B_{t+\Delta t}^j - B_t^j) + O(\Delta t^{\frac{3}{2}}).\label{eq:sdetaylor}
\end{align}
Since $B_{t+\Delta t}^j - B_t^j$ follows a zero-mean Gaussian and is independent of the state and the action,
\begin{equation}
\mathbb{E}_{p_{\pi}}\Bigg[\sum_{j=1}^{m} \sigma^{ij}(s_t, A_t)(B_{t+\Delta t}^j - B_t^j)\Bigg] 
= \sum_{j=1}^{m} \mathbb{E}_{\pi}\left[\sigma^{ij}(s_t, A_t) \right] \mathbb{E}\left[B_{t+\Delta t}^j - B_t^j\right] 
= 0. \notag 
\end{equation}
By taking the expectation of both sides of \eqref{eq:sdetaylor} and then letting $\Delta t\rightarrow 0$, the terms in $O(\Delta t^{\frac{3}{2}})$ vanish, and we obtain \eqref{eq:claim1}.
For the second part of the claim, we begin by considering the product:
\begin{equation}
\smalleq{\begin{aligned}
&(S_{t+\Delta t}^i - s_t^i)(S_{t+\Delta t}^j - s_t^j) \\
& = \mu^i(s_t, A_t)\mu^j(s_t, A_t)\Delta t^2 + \sum_{k=1}^{m} \sum_{l=1}^{m} \sigma^{ik}(s_t, A_t) \sigma^{jl}(s_t, A_t) (B_{t+\Delta t}^k - B_t^k)(B_{t+\Delta t}^l - B_t^l)\\
&\quad + \mu^i(s_t, A_t)\Delta t \sum_{k=1}^{m} \sigma^{jk}(s_t, A_t)(B_{t+\Delta t}^k - B_t^k) + \mu^j(s_t, A_t)\Delta t \sum_{k=1}^{m} \sigma^{ik}(s_t, A_t)(B_{t+\Delta t}^k - B_t^k)+ O(\Delta t^{\frac{3}{2}}).
\end{aligned}}
\label{eq:product}
\end{equation}
and then take the expectation of both sides. Using the fact
\begin{align}
\mathbb{E}\left[(B_{t+\Delta t}^k - B_t^k)(B_{t+\Delta t}^l - B_t^l)\right] = \delta_{kl}\Delta t, \notag
\end{align}
where $\delta_{kl}=1$ if $k=l$ and $\delta_{kl}=0$ otherwise, we can take the expectation of both sides of equation~\eqref{eq:product}
and then let $\Delta t\rightarrow 0$, which yields \eqref{eq:claim2}.
\end{proof}
\begin{remark}
Note that since $S_{t+\Delta t}$ in equations~(5) and~(6) is sampled under the policy $\pi$, 
our method is not applicable to off-policy settings such as value iteration or Q-learning.
This limitation reflects a fundamental distinction: our formulation relies on the HJB equation under a fixed policy (i.e., policy evaluation),
rather than the classical HJB equation involving maximization over all policies.
\end{remark}

From Proposition~\ref{thm:coef}, the HJB equation \eqref{eq:ehjb} can be reformulated as
\begin{equation*}
\smalleq{\begin{aligned}
V^{\pi}(s_t) &= \frac{1}{\gamma} \lim_{\Delta t \to 0}\mathbb{E}_{p_{\pi}} \Bigg[
    \rho(s_t, A_t) 
    + \sum_{i=1}^{n}  \frac{S_{t+\Delta t}^i - s_t^i}{\Delta t}
      \frac{\partial V^\pi(s)}{\partial s^i} \bigg|_{s_t}
    \\ &\qquad\qquad
    + \frac{1}{2} \sum_{i=1}^{n} \sum_{j=1}^{n} 
     \frac{(S_{t+\Delta t}^i - s_t^i)(S_{t+\Delta t}^j - s_t^j)}{\Delta t}
    \frac{\partial^2 V^\pi(s)}{\partial s^i \partial s^j} \bigg|_{s_t}
\Bigg].
\end{aligned}}
\end{equation*}

As we have rearranged the HJB equation so that the argument of expectation does not depend on the model, $\mu$ and $\sigma$, we can construct a temporal‑difference update directly from it. We refer to this update as \emph{differential temporal difference (dTD)} and expect it to be particularly effective when the observation interval $\Delta t$ is small.

\begin{figure}[t]
\centering
\begin{minipage}[t]{0.53\linewidth}
    \vspace{0pt}
    \centering
    \includegraphics[width=0.95\linewidth,clip,trim=0 10mm 0 0]{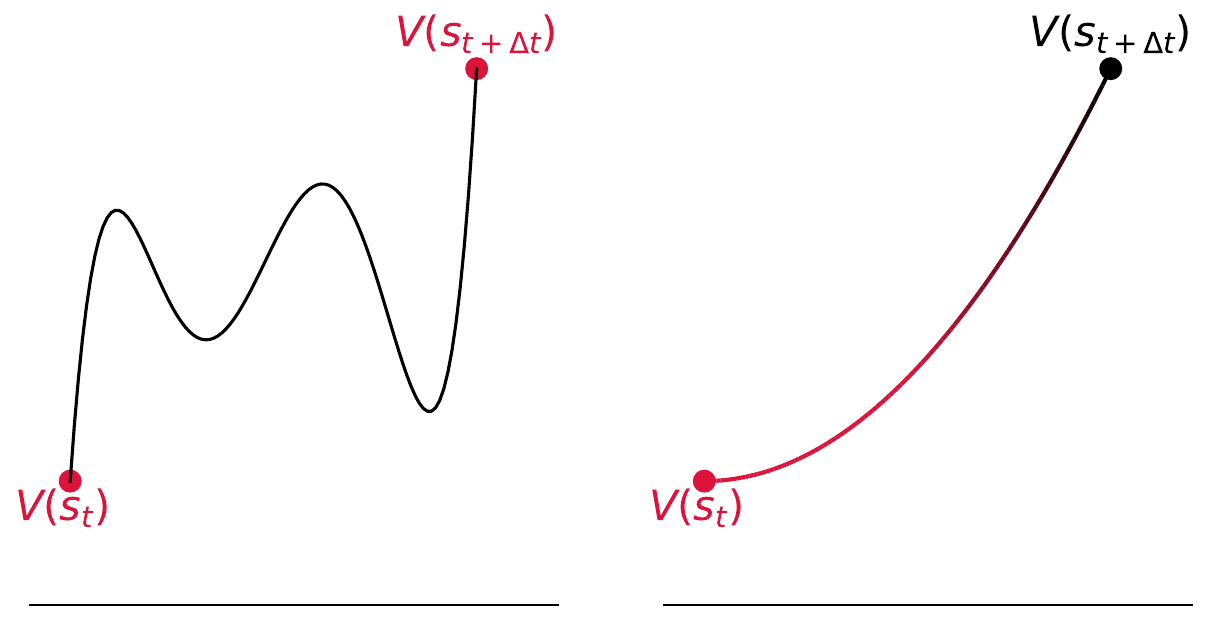} 
    \hspace*{1mm}
    {\small typical TD}
    \hspace{20mm}
    {\small proposed dTD}
\end{minipage}
\hfill
\begin{minipage}[t]{0.43\linewidth}
    \vspace{0pt}
    \caption{Qualitative difference between the typical TD method and the proposed dTD method; \textcolor{myredfig1}{the objects in red indicate what is adjusted} by each temporal difference. (\emph{Left}) In the typical TD method, \textcolor{myredfig1}{the values of $\hat{V}$} are adjusted to minimize the TD error. (\emph{Right}) In the dTD method, \textcolor{myredfig1}{the gradient and the second derivative of $\hat{V}$} at $s_t$ are adjusted to minimize the dTD error.}
    \label{fig:dTD}
\end{minipage}
\end{figure}

\begin{definition}[differential temporal difference]
Let $\Delta t > 0$ be a time step and $\widehat{V}$ denote an estimated value function.
The dTD is defined as:
\begin{equation}
\smalleq{\begin{aligned}
\mathrm{dTD} := \frac{1}{\gamma} \Bigg(
    & \rho(s_t, a_t)
    + \sum_{i=1}^{n} \frac{s_{t+\Delta t}^i - s_t^i}{\Delta t} 
        \frac{\partial \widehat{V}(s)}{\partial s^i} \bigg|_{s_t} \\
    & + \frac{1}{2} \sum_{i=1}^{n} \sum_{j=1}^{n} 
        \frac{(s_{t+\Delta t}^i - s_t^i)(s_{t+\Delta t}^j - s_t^j)}{\Delta t} 
        \frac{\partial^2 \widehat{V}(s)}{\partial s^i \partial s^j} \bigg|_{s_t}
\Bigg) - \widehat{V}(s_t).
\end{aligned}}
\label{eq:dtd}
\end{equation}
\end{definition}

As illustrated in Figure~\ref{fig:dTD}, unlike conventional TD methods based on transition kernels, dTD encourages the learning of a smooth value function by incorporating the continuity of the state space, even under sample-based approximation.
%
We note that the version of dTD for ODE systems can be recovered by simply removing the term corresponding to the diffusion coefficient $\sigma$.

\subsection{Convergence Analysis}
To analyze the convergence, we first define the key operators 
for the fixed-policy HJB equation.

\begin{definition}[HJB operator under a fixed policy] 
\label{def:HJB_operator}
Given a stationary Markov policy $\pi:S\!\to\!A$, discount rate $\gamma>0$,
drift $\mu:S\times A\!\to\!\mathbb R^{n}$, diffusion
$\sigma:S\times A\!\to\!\mathbb R^{n\times m}$, and instantaneous reward
rate $\rho:S\times A\!\to\!\mathbb R$, the \emph{HJB operator under a fixed policy} $T$ maps any function
$V:S\!\to\!\mathbb R$ to
\begin{equation*}
\smalleq{\begin{aligned}
(TV)(s_t) &:= \frac{1}{\gamma}\mathbb E_{\pi}\Bigg[\rho(s_t, A_t) + \underset{i=1}{\overset{n}{\sum}}\mu^i(s_t, A_t)\frac{\partial V(s)}{\partial s^i}\bigg|_{s_t}
\!
+ \frac{1}{2}\underset{i=1}{\overset{n}{\sum}}\underset{j=1}{\overset{n}{\sum}}[\sigma(s_t, A_t)\sigma^\top(s_t, A_t)]^{ij}\frac{\partial^2 V(s)}{\partial s^i \partial s^j}\bigg|_{s_t} \Bigg].
\end{aligned}}
\end{equation*}
\end{definition}

\begin{definition}[Infinitesimal Generator]
\label{def:generator}
Given the same policy $\pi$ and system dynamics $(\mu, \sigma)$ as in Definition \ref{def:HJB_operator}, the \emph{infinitesimal generator} $\mathcal{L}^\pi$ maps any $C^2$ function $V:S\!\to\!\mathbb R$ to
\begin{equation*}
\smalleq{\begin{aligned}
(\mathcal{L}^\pi V)(s) &:= \mathbb E_{\pi}\Bigg[ \underset{i=1}{\overset{n}{\sum}}\mu^i(s, A)\frac{\partial V(s)}{\partial s^i}
\!
+ \frac{1}{2}\underset{i=1}{\overset{n}{\sum}}\underset{j=1}{\overset{n}{\sum}}[\sigma(s, A)\sigma^\top(s, A)]^{ij}\frac{\partial^2 V(s)}{\partial s^i \partial s^j} \Bigg].
\end{aligned}}
\end{equation*}
\end{definition}

Unlike the Bellman operator, whose convergence is guaranteed by its contraction property, the HJB operator involves unbounded differential operators and is generally not a contraction. Therefore, we establish convergence of the iterative scheme, i.e., 

\begin{align*}
V_{k+1}=V_{k}+\eta_{k}(TV_{k}-V_{k})
\end{align*}

by analyzing its continuous-time limit,

\begin{align*}
\frac{V_{k+1}-V_{k}}{\eta_{k}}=TV_{k}-V_{k}\overset{\eta_k\rightarrow 0}{\Longrightarrow}\frac{\partial V(t)}{\partial t}=TV-V,
\end{align*}

under the idealized setting that the function $V$ itself can be directly manipulated and updated (rather than updated via a finite set of parameters). The convergence analysis thus relies on showing that these dynamics asymptotically stabilize to the unique fixed point $V$ satisfying $TV=V$. Since we focus on the HJB equation for a fixed policy, which lacks the $\max$ operator, this problem reduces to the analysis of a linear elliptic PDE. Using Definitions \ref{def:HJB_operator} and \ref{def:generator}, and defining the expected reward rate as $\bar{\rho}(s) := \mathbb E_{\pi}[\rho(s, A)]$, the equation $V = TV$ is equivalent to:

\begin{align*}
(\gamma I-\mathcal{L}^\pi)V(s)=\bar{\rho}(s).
\end{align*}

We show that the standard analysis for such PDEs, based on the Lax-Milgram theorem, can be applied to the bilinear form associated with the operator $(\gamma I - \mathcal{L}^\pi)$ to guarantee convergence to this unique fixed point. (We leave the analysis of errors from function approximation, which remains a challenging open problem even for standard Bellman operators, to future work.)

\begin{assumption}\label{as:hjb_assumptions}
We assume the following conditions hold:
\begin{enumerate}
\item \textbf{(Domain)} The state space $S$ is a bounded, connected open subset of $\mathbb R^n$ with a smooth boundary $\partial S$. We conduct our analysis in the Sobolev space $H^1(S)$.
\item \textbf{(Reflecting Boundary)} We assume the system satisfies a Neumann condition on $\partial S$:
\begin{equation*}
n(s)\cdot \left(D(s)\,\nabla V(s)\right)=0 \quad \text{for } s\in \partial S,
\end{equation*}
where $D(s) = \mathbb E_{\pi}[\sigma(s, A)\sigma^\top(s, A)]$, and $n(s)$ is the outward unit normal on $\partial S$.
\item \textbf{(Coefficients)} The policy-averaged reward $\bar{\rho}(s) := \mathbb E_{\pi}[\rho(s, A)]$, drift $\bar{\mu}(s) := \mathbb E_{\pi}[\mu(s, A)]$, and diffusion matrix $D(s)$ are bounded and Lipschitz continuous on $\bar{S}$.
\item \textbf{(Uniform Ellipticity)} The policy-averaged diffusion matrix $D(s)$ is symmetric and uniformly elliptic. That is, there exists a constant $\alpha > 0$ such that for all $s \in \bar{S}$ and $\xi \in \mathbb R^n$:
\[
\xi^\top D(s)\, \xi \ge \alpha \|\xi\|^2.
\]
\item \textbf{(Discount Factor)} The discount rate $\gamma>0$ is sufficiently large to ensure the coercivity of the bilinear form associated with the operator $(\gamma I-\mathcal{L}^\pi)$.
\end{enumerate}
\end{assumption}

Using these assumptions, we first formally establish the existence and uniqueness of the solution $V^\pi$ to the fixed-point equation $(\gamma I - \mathcal{L}^\pi)V = \bar{\rho}$. This solution $V^\pi$ serves as the unique fixed point for our dynamics. A proof is deferred to Appendix~\ref{app:convergence}.

\begin{lemma}[Existence and Uniqueness of the Fixed Point]
\label{lem:existence_uniqueness}
Under Assumption \ref{as:hjb_assumptions}, the linear elliptic PDE
\begin{align*}
(\gamma I - \mathcal{L}^\pi)V(s) = \bar{\rho}(s) \quad \text{for } s \in S
\end{align*}
admits a unique weak solution $V^\pi \in H^1(S)$.
\end{lemma}

\begin{proposition}[Exponential Stability of the Dynamics]
\label{prop:convergence}
Let $V^\pi \in H^1(S)$ be the unique fixed point from Lemma \ref{lem:existence_uniqueness}. Under Assumption \ref{as:hjb_assumptions}, the solution $V(t)$ of the continuous-time dynamics
$$\frac{\partial V(t)}{\partial t}=TV(t)-V(t)$$
converges exponentially fast to $V^\pi$ in the $L^2(S)$ norm. Specifically, there exists a constant $\lambda > 0$ such that for any initial condition $V(0) \in L^2(S)$:
$$\|V(t) - V^\pi\|_{L^2(S)} \le e^{-\lambda t}\|V(0) - V^\pi\|_{L^2(S)} .$$
\end{proposition}

We comment on the interpretation and practicality of the assumptions in Assumption~\ref{as:hjb_assumptions} as follows.
\textbf{(1) Domain.} The state space $S$ is bounded, which reflects physical or operational limits, such as joint limits or maximum velocity, making it a standard and natural modeling choice. \textbf{(2) Reflecting Boundary.} The Neumann condition ($n\cdot(D\nabla V)=0$) models a reflecting boundary. This is physically natural for state-constrained systems that cannot exit $S$. \textbf{(3) Coefficients.} Bounded/Lipschitz $\bar\mu$, $D$, and $\bar\rho$ are classical regularity conditions ensuring well-posed SDEs and boundedness of the weak form; they are routinely satisfied by smooth dynamics and regular policies.
\textbf{(4) Uniform ellipticity.} This implies the presence of non-degenerate random excitation (noise) in all directions. This is physically plausible for complex systems where environmental or internal noise can perturb the state in any dimension
\textbf{(5) Discount.} The discount $\gamma>0$ supplies a zero-order “anchor” that absorbs drift terms and yields a spectral gap; in practice, one can take $\gamma$ large enough (with the usual horizon–myopia trade-off) to guarantee coercivity.



\section{Method}

This section outlines our method for applying dTD in deep reinforcement learning. Because dTD relies on function approximation, we restrict our attention to the deep RL regime, representing the value function with a neural network. A concise pseudocode listing is provided in Appendix~\ref{app:algorithms}; here we explain the loss formulation and the $\beta$‑dTD stabilization strategy.

\subsection{Loss function}
In deep RL, TD methods typically use a fixed target, known as the TD target, $r(s, t) + \gamma_{\text{discrete}} V(s_{t+1})$, as the teacher and aim to approximate the prediction $V(s_t)$ by minimizing the squared error between them.
Although it may seem natural, by analogy with classical TD, to treat the $V(s_t)$ term in the dTD~\eqref{eq:dtd} as the prediction and regard the remaining terms as the dTD target, this is in fact unnecessary. As shown in Appendix~\ref{app:ito}, the terms that appear in \eqref{eq:ehjb} and thus in \eqref{eq:dtd} are derived through a series of transformations, and thus in dTD we no longer interpret the terms other than $V(s_t)$ as a low-variance estimate of the Bellman error. Consequently, the split between prediction and target is not unique.

We examine two different ways of defining the prediction and target from the rhs of \eqref{eq:dtd}. 
\begin{itemize}[topsep=0pt,leftmargin=1em]
\item As a baseline, we first consider a naive formulation following the typical TD-style decomposition, that is, we treat $V(s_t)$ as the prediction. We refer to this approach as \textbf{naive-dTD}.
\item On the other hand, inspired by the Taylor expansion of the Bellman equation, we treat the terms involving the derivative of $V(s_t)$ as the prediction and regard the rest as the target. We term such a more motivated parametrization simply as \textbf{dTD} hereafter.
\end{itemize}


These two variants are summarized in Table~\ref{table:target}. As introduced in the next section, we empirically found that dTD performed significantly better than naive-dTD.

\begin{table*}[tb] 
\centering
\caption{
Comparison of target and prediction terms in TD methods.
Here, $\Delta s_t^i := s_{t+\Delta t}^i - s_t^i$ denotes the $i$-th component of the state transition over a small time interval $\Delta t$.
}
\begin{tabular}{ccc}
\toprule
 & \text{Target} & \text{Prediction} \\
\midrule
\text{TD} & $r(s, t) + \gamma_{\text{discrete}} V(s_{t+1})$ & $V(s_t)$ \\
\cmidrule(lr){1-3}
\shortstack{\text{naive-dTD}\\{}\\{}} & \shortstack{$\rho(s_t, a_t) + \displaystyle\sum_{i=1}^{n}\frac{\Delta s_t^i}{\Delta t} \frac{\partial V(s)}{\partial s^i}\bigg|_{s_t}$ \\ \hspace{3em}$+ \frac{1}{2} \displaystyle \sum_{i=1}^{n}\sum_{j=1}^{n} \frac{\Delta s_t^i\Delta s_t^j}{\Delta t} \frac{\partial^2 V(s)}{\partial s^i \partial s^j}\bigg|_{s_t}$}
 & \shortstack{$\gamma V(s_t)$\\{}\\{}} \\
 \cmidrule(lr){1-3}
\shortstack{\text{dTD}\\{}\\{}} & \shortstack{$-\rho(s_t, a_t) + \gamma V(s_{t+\Delta t})$\\{}\\{}} & \shortstack{$\displaystyle \sum_{i=1}^{n}\frac{\Delta s_t^i}{\Delta t} \frac{\partial V(s)}{\partial s^i}\bigg|_{s_t}$ \\ $+ \frac{1}{2} \displaystyle \sum_{i=1}^{n} \sum_{j=1}^{n} \frac{\Delta s_t^i\Delta s_t^j}{\Delta t} \frac{\partial^2 V(s)}{\partial s^i \partial s^j}\bigg|_{s_t}$} \\
\bottomrule
\end{tabular}
\label{table:target}
\end{table*}

\subsection{Hybrid scheme for stabilizing dTD}

Although Proposition~\ref{prop:convergence} establishes a convergence analysis, it relies on an idealized setting and does not cover the practical instability that can arise from function approximation errors in deep RL. Consequently, we cannot a priori guarantee that plain dTD operates stably and efficiently in practice.

To make the critic update more robust, we linearly combine the classical TD error with the dTD error, using weights $1-\beta$ and $\beta$, respectively; we call the resulting update $\beta$‑dTD. The TD part supplies the empirical stability that underpins most deep RL algorithms, whereas the dTD part injects gradient information from the continuous dynamics, accelerating learning when the underlying assumptions are approximately satisfied. We hypothesize that $\beta$‑dTD can strike a balance to stabilize learning progress and potentially improve convergence behavior in practice.

\subsection{Efficient Computation of the dTD Loss}
Since equation~\eqref{eq:dtd} involves the Hessian, it may seem that $O(n^2)$ (where $n$ is the dimension of the observation space) computations are required. However, by rearranging the order of calculations, such as using
\begin{equation}
\Bigg\langle \Delta s_t,\ \frac{\partial^2 V(s)}{\partial s^2}\bigg|_{s_{t}}\Delta s_t\Bigg\rangle = \Bigg\langle \Delta s_t,\ \frac{\partial}{\partial s}\Bigg\langle\frac{\partial V}{\partial s}, \Delta s_t\Bigg\rangle\bigg|_{s_{t}}\Bigg\rangle, \notag
\end{equation}
we can avoid directly calculating the Hessian and achieve a computation complexity of $O(n)$.

\section{Experiments}
\subsection{Modification for discrete environment compatibility}
In our theoretical framework, we work with continuous rewards (i.e., reward rate function) and a specific form of the discount ratio $e^{-\gamma}$, which is not directly compatible with the discrete discount ratio $\gamma_{\text{discrete}}$. To address this, we adjusted the reward and discount ratio following the same approach discussed
in \citet{tbo:19}.
The continuous reward formulation can be approximated by:
\begin{align}
\int_{0}^{\infty} \text{e}^{-\gamma t} \rho(s_t, a_t)dt \approx \underset{k=0}{\overset{\infty}{\sum}}\text{e}^{(-\gamma\Delta t)k} \rho(s_{k\Delta t}, a_{k\Delta t})\Delta t.\notag
\end{align}
In this approximation, $\rho(s_{t}, a_t)\Delta t$ corresponds to the discrete reward $r$, and $e^{-\gamma \Delta t}$ corresponds to the discrete discount ratio $\gamma_{\text{discrete}}$. Thus, we can establish the following relationship:
\[
\rho(s_t, a_t) = \frac{r(s_t, a_t)}{\Delta t} \quad \text{and} \quad \gamma = -\frac{1}{\Delta t}\log({\gamma_{\text{discrete}}}).
\]
This adjustment ensures that the observed discrete rewards are properly scaled to align with the continuous reward formulation used in the dTD method. With this scaling, dTD can be computed as
\begin{equation*}
\smalleq{\begin{aligned}
&\text{dTD Target}: - r(s_t, a_t) - \log({\gamma_{\text{discrete}}}) V(s_{t+\Delta t}) \quad\text{and}\quad \\
&\text{dTD Prediction}:\ 
    \sum_{i=1}^{n}(s_{t+\Delta t}^i - s_t^i)\frac{\partial V(s)}{\partial s_i}\bigg|_{s_{t}} 
    + \frac{1}{2} \sum_{i=1}^{n} \sum_{j=1}^{n} 
    (s_{t+\Delta t}^i - s_t^i)(s_{t+\Delta t}^j - s_t^j)
    \frac{\partial^2 V(s)}{\partial s_i \partial s_j}\bigg|_{s_{t}}.
\end{aligned}}
\end{equation*}

\subsection{Experiment design}
\paragraph{Environment}
We conducted experiments with the Brax\footnote{\url{https://github.com/google/brax}} library \citep{ffr:21} in the following environments: Hopper, HalfCheetah, Ant and Humanoid. Each environment provides a mid- to high-dimensional state space, with the number of state components varying across environments: Hopper (11 dimensions), HalfCheetah (17 dimensions), Ant (27 dimensions) and Humanoid (244 dimensions). In each environment, at every step, we perturbed each state component by adding noise in the form of
\begin{equation*}
    s_i \leftarrow s_i + \text{coef}\times|s_i|\times\text{noise},
\end{equation*}
where $\text{noise}\sim\mathcal N(0, 1)$, and tested for three values of $\text{coef} = 0.00, 0.01, 0.05$. By adding this process noise, we aim to simulate with SDE systems, with the case of $\text{coef} = 0.00$ representing the limit case corresponding to ODEs. The specific time-step values used for each environment, which are not directly used in the learning process but are important to ensure they are small enough, are: Hopper: $\Delta t = 0.008$, HalfCheetah: $\Delta t = 0.05$, Ant: $\Delta t = 0.05$ and Humanoid: $\Delta t = 0.015$.

\paragraph{Baseline}
Various methods have been developed specifically for settings such as ODE \citep{tbo:19}, LQR \citep{vc:17}, or time-dependent Q functions with finite horizons \citep{jz:23}, but are often incompatible with the current deep RL framework \citep{ksy:21} or rely on model-based assumptions \citep{mb:97}, making them unsuitable for comparison with our proposed method.
We chose to use standard TD methods as baselines and experimented with TD, $\beta$-naive-dTD, and $\beta$-dTD using PPO \citep{swd:17}.

\paragraph{Hyperparameter tuning}
 For hyperparameter tuning, we applied the DEHB \citep{amh:21}, a multi-fidelity method that is currently considered the most effective method in RL \citep{elr:23}. While we performed hyperparameter tuning for the standard PPO algorithm as well, we also reference the official tuning results from \citet{ffr:21} for fair comparison. Additional details about the hyperparameter search space can be found in Appendix~\ref{app:implementation}.

\begin{figure}[H]
\centering

\pgfplotstableread[col sep=comma]{data/hopper_baseline_noise_lvl000_1747264299.csv}\hopbasezero
\pgfplotstableread[col sep=comma]{data/hopper_dtd_noise_lvl000_1747344528.csv}\hopdtdzero
\pgfplotstableread[col sep=comma]{data/hopper_shjb_noise_lvl000_1747139480.csv}\hopnaivezero

\pgfplotstableread[col sep=comma]{data/hopper_baseline_noise_lvl001_1747283609.csv}\hopbaseone
\pgfplotstableread[col sep=comma]{data/hopper_dtd_noise_lvl001_1747324235.csv}\hopdtdone
\pgfplotstableread[col sep=comma]{data/hopper_shjb_noise_lvl001_1747140229.csv}\hopnaiveone

\pgfplotstableread[col sep=comma]{data/hopper_baseline_noise_lvl005_1747269907.csv}\hopbasefive
\pgfplotstableread[col sep=comma]{data/hopper_dtd_noise_lvl005_1747346096.csv}\hopdtdfive
\pgfplotstableread[col sep=comma]{data/hopper_shjb_noise_lvl005_1747140321.csv}\hopnaivefive

\pgfplotstableread[col sep=comma]{data/halfcheetah_baseline_noise_lvl000_1747288373.csv}\hcbasezero
\pgfplotstableread[col sep=comma]{data/halfcheetah_dtd_noise_lvl000_1747246341.csv}\hcdtdzero
\pgfplotstableread[col sep=comma]{data/halfcheetah_shjb_noise_lvl000_1747139455.csv}\hcnaivezero

\pgfplotstableread[col sep=comma]{data/halfcheetah_baseline_noise_lvl001_1747277975.csv}\hcbaseone
\pgfplotstableread[col sep=comma]{data/halfcheetah_dtd_noise_lvl001_1747327875.csv}\hcdtdone
\pgfplotstableread[col sep=comma]{data/halfcheetah_shjb_noise_lvl001_1747140207.csv}\hcnaiveone

\pgfplotstableread[col sep=comma]{data/halfcheetah_baseline_noise_lvl005_1747284147.csv}\hcbasefive
\pgfplotstableread[col sep=comma]{data/halfcheetah_dtd_noise_lvl005_1747327047.csv}\hcdtdfive
\pgfplotstableread[col sep=comma]{data/halfcheetah_shjb_noise_lvl005_1747140300.csv}\hcnaivefive

\pgfplotstableread[col sep=comma]{data/ant_baseline_noise_lvl000_1747267684.csv}\antbasezero
\pgfplotstableread[col sep=comma]{data/ant_dtd_noise_lvl000_1747186365.csv}\antdtdzero
\pgfplotstableread[col sep=comma]{data/ant_shjb_noise_lvl000_1747143643.csv}\antnaivezero

\pgfplotstableread[col sep=comma]{data/ant_baseline_noise_lvl001_1747266857.csv}\antbaseone
\pgfplotstableread[col sep=comma]{data/ant_dtd_noise_lvl001_1747340331.csv}\antdtdone
\pgfplotstableread[col sep=comma]{data/ant_shjb_noise_lvl001_1747140185.csv}\antnaiveone

\pgfplotstableread[col sep=comma]{data/ant_baseline_noise_lvl005_1747277058.csv}\antbasefive
\pgfplotstableread[col sep=comma]{data/ant_dtd_noise_lvl005_1747353835.csv}\antdtdfive
\pgfplotstableread[col sep=comma]{data/ant_shjb_noise_lvl005_1747140279.csv}\antnaivefive

\pgfplotstableread[col sep=comma]{data/humanoid_baseline_noise_lvl000_1738223399.csv}\humbasezero
\pgfplotstableread[col sep=comma]{data/humanoid_dtd_noise_lvl000_1738228527.csv}\humdtdzero
\pgfplotstableread[col sep=comma]{data/humanoid_shjb_noise_lvl000_1747140253.csv}\humnaivezero

\pgfplotstableread[col sep=comma]{data/humanoid_baseline_noise_lvl001_1738223407.csv}\humbaseone
\pgfplotstableread[col sep=comma]{data/humanoid_dtd_noise_lvl001_1738235634.csv}\humdtdone
\pgfplotstableread[col sep=comma]{data/humanoid_shjb_noise_lvl001_1747140252.csv}\humnaiveone

\pgfplotstableread[col sep=comma]{data/humanoid_baseline_noise_lvl005_1738223416.csv}\humbasefive
\pgfplotstableread[col sep=comma]{data/humanoid_dtd_noise_lvl005_1738043094.csv}\humdtdfive
\pgfplotstableread[col sep=comma]{data/humanoid_shjb_noise_lvl005_1747140339.csv}\humnaivefive

\resizebox{\textwidth}{!}{\begin{tikzpicture}
\begin{groupplot}[
  group style={
    group size=3 by 4,
    horizontal sep=0.6cm,
    vertical sep=0.45cm,
  },
  width=4.2cm,
  height=3.2cm,
  title style={font=\tiny, yshift=-0.8em},
  tick label style={
    font=\tiny,
    inner sep=0pt,
    xshift=-0.1em,
    yshift=-0.1em
  },
  legend style={
    draw=none,
    fill=none,
    font=\tiny,
    at={(0.03,0.97)},  
    anchor=north west,  
    legend image code/.code={
      \draw[##1, line width=0.5pt] (0cm,0cm) -- (0.2cm,0cm);
    }
  },
  every x tick scale label/.style={
        xshift=7em,
        yshift=-0.7em
  },
  every y tick scale label/.style={
        xshift=-0.6em,
        yshift=4.8em
  },
]
\nextgroupplot[title={Hopper (noise: 0.00)},
  ylabel={\text{Episodic return}},
  ylabel style={
    font=\tiny,
    yshift=-0.2em  
  },
  scaled y ticks=base 10:-3,
  ]
    \addplot+[mark=none, myred] table[x=x, y=mean] {\hopdtdzero};
    \addplot[name path=upperb1, draw=none] table[x=x, y expr=\thisrow{mean}+\thisrow{std}] {\hopdtdzero};
    \addplot[name path=lowerb1, draw=none] table[x=x, y expr=\thisrow{mean}-\thisrow{std}] {\hopdtdzero};
    \addplot[fill=myred, fill opacity=0.2] fill between[of=upperb1 and lowerb1];

    \addplot+[mark=none, myblue, dash dot] table[x=x, y=mean] {\hopnaivezero};
    \addplot[name path=upperb1, draw=none] table[x=x, y expr=\thisrow{mean}+\thisrow{std}] {\hopnaivezero};
    \addplot[name path=lowerb1, draw=none] table[x=x, y expr=\thisrow{mean}-\thisrow{std}] {\hopnaivezero};
    \addplot[fill=myblue, fill opacity=0.2] fill between[of=upperb1 and lowerb1];
    
    \addplot+[mark=none, mygray, dashed] table[x=x, y=mean] {\hopbasezero};
    \addplot[name path=upperg1, draw=none] table[x=x, y expr=\thisrow{mean}+\thisrow{std}] {\hopbasezero};
    \addplot[name path=lowerg1, draw=none] table[x=x, y expr=\thisrow{mean}-\thisrow{std}] {\hopbasezero};
    \addplot[fill=mygray, fill opacity=0.2] fill between[of=upperg1 and lowerg1];

\nextgroupplot[title={Hopper (noise: 0.01)}, scaled y ticks=base 10:-3,]
    \addplot+[mark=none, myred] table[x=x, y=mean] {\hopdtdone};
    \addplot[name path=upperb1, draw=none] table[x=x, y expr=\thisrow{mean}+\thisrow{std}] {\hopdtdone};
    \addplot[name path=lowerb1, draw=none] table[x=x, y expr=\thisrow{mean}-\thisrow{std}] {\hopdtdone};
    \addplot[fill=myred, fill opacity=0.2] fill between[of=upperb1 and lowerb1];

    \addplot+[mark=none, myblue, dash dot] table[x=x, y=mean] {\hopnaiveone};
    \addplot[name path=upperb1, draw=none] table[x=x, y expr=\thisrow{mean}+\thisrow{std}] {\hopnaiveone};
    \addplot[name path=lowerb1, draw=none] table[x=x, y expr=\thisrow{mean}-\thisrow{std}] {\hopnaiveone};
    \addplot[fill=myblue, fill opacity=0.2] fill between[of=upperb1 and lowerb1];
    
    \addplot+[mark=none, mygray, dashed] table[x=x, y=mean] {\hopbaseone};
    \addplot[name path=upperg1, draw=none] table[x=x, y expr=\thisrow{mean}+\thisrow{std}] {\hopbaseone};
    \addplot[name path=lowerg1, draw=none] table[x=x, y expr=\thisrow{mean}-\thisrow{std}] {\hopbaseone};
    \addplot[fill=mygray, fill opacity=0.2] fill between[of=upperg1 and lowerg1];

\nextgroupplot[title={Hopper (noise: 0.05)}, scaled y ticks=base 10:-3,]
    \addplot+[mark=none, myred] table[x=x, y=mean] {\hopdtdfive};
    \addplot[name path=upperb1, draw=none] table[x=x, y expr=\thisrow{mean}+\thisrow{std}] {\hopdtdfive};
    \addplot[name path=lowerb1, draw=none] table[x=x, y expr=\thisrow{mean}-\thisrow{std}] {\hopdtdfive};
    \addplot[fill=myred, fill opacity=0.2] fill between[of=upperb1 and lowerb1];

    \addplot+[mark=none, myblue, dash dot] table[x=x, y=mean] {\hopnaivefive};
    \addplot[name path=upperb1, draw=none] table[x=x, y expr=\thisrow{mean}+\thisrow{std}] {\hopnaivefive};
    \addplot[name path=lowerb1, draw=none] table[x=x, y expr=\thisrow{mean}-\thisrow{std}] {\hopnaivefive};
    \addplot[fill=myblue, fill opacity=0.2] fill between[of=upperb1 and lowerb1];
    
    \addplot+[mark=none, mygray, dashed] table[x=x, y=mean] {\hopbasefive};
    \addplot[name path=upperg1, draw=none] table[x=x, y expr=\thisrow{mean}+\thisrow{std}] {\hopbasefive};
    \addplot[name path=lowerg1, draw=none] table[x=x, y expr=\thisrow{mean}-\thisrow{std}] {\hopbasefive};
    \addplot[fill=mygray, fill opacity=0.2] fill between[of=upperg1 and lowerg1];
    
\nextgroupplot[title={Halfcheetah (noise: 0.00)},
    ylabel={\text{Episodic return}},
    ylabel style={
      font=\tiny,
      yshift=-0.2em  
    },
    scaled y ticks=base 10:-4,  
    ytick={0, 10000, 20000},   
    yticklabels={0, 1, 2},
  ]
    \addplot+[mark=none, myred] table[x=x, y=mean] {\hcdtdzero};
    \addplot[name path=upperb1, draw=none] table[x=x, y expr=\thisrow{mean}+\thisrow{std}] {\hcdtdzero};
    \addplot[name path=lowerb1, draw=none] table[x=x, y expr=\thisrow{mean}-\thisrow{std}] {\hcdtdzero};
    \addplot[fill=myred, fill opacity=0.2] fill between[of=upperb1 and lowerb1];

    \addplot+[mark=none, myblue, dash dot] table[x=x, y=mean] {\hcnaivezero};
    \addplot[name path=upperb1, draw=none] table[x=x, y expr=\thisrow{mean}+\thisrow{std}] {\hcnaivezero};
    \addplot[name path=lowerb1, draw=none] table[x=x, y expr=\thisrow{mean}-\thisrow{std}] {\hcnaivezero};
    \addplot[fill=myblue, fill opacity=0.2] fill between[of=upperb1 and lowerb1];
    
    \addplot+[mark=none, mygray, dashed] table[x=x, y=mean] {\hcbasezero};
    \addplot[name path=upperg1, draw=none] table[x=x, y expr=\thisrow{mean}+\thisrow{std}] {\hcbasezero};
    \addplot[name path=lowerg1, draw=none] table[x=x, y expr=\thisrow{mean}-\thisrow{std}] {\hcbasezero};
    \addplot[fill=mygray, fill opacity=0.2] fill between[of=upperg1 and lowerg1];

\nextgroupplot[title={Halfcheetah (noise: 0.01)},
    scaled y ticks=base 10:-4,  
    ytick={0, 10000, 20000},   
    yticklabels={0, 1, 2},
  ]
    \addplot+[mark=none, myred] table[x=x, y=mean] {\hcdtdone};
    \addplot[name path=upperb1, draw=none] table[x=x, y expr=\thisrow{mean}+\thisrow{std}] {\hcdtdone};
    \addplot[name path=lowerb1, draw=none] table[x=x, y expr=\thisrow{mean}-\thisrow{std}] {\hcdtdone};
    \addplot[fill=myred, fill opacity=0.2] fill between[of=upperb1 and lowerb1];

    \addplot+[mark=none, myblue, dash dot] table[x=x, y=mean] {\hcnaiveone};
    \addplot[name path=upperb1, draw=none] table[x=x, y expr=\thisrow{mean}+\thisrow{std}] {\hcnaiveone};
    \addplot[name path=lowerb1, draw=none] table[x=x, y expr=\thisrow{mean}-\thisrow{std}] {\hcnaiveone};
    \addplot[fill=myblue, fill opacity=0.2] fill between[of=upperb1 and lowerb1];
    
    \addplot+[mark=none, mygray, dashed] table[x=x, y=mean] {\hcbaseone};
    \addplot[name path=upperg1, draw=none] table[x=x, y expr=\thisrow{mean}+\thisrow{std}] {\hcbaseone};
    \addplot[name path=lowerg1, draw=none] table[x=x, y expr=\thisrow{mean}-\thisrow{std}] {\hcbaseone};
    \addplot[fill=mygray, fill opacity=0.2] fill between[of=upperg1 and lowerg1];

\nextgroupplot[title={Halfcheetah (noise: 0.05)},
    scaled y ticks=base 10:-4,  
    ytick={0, 10000, 20000},   
    yticklabels={0, 1, 2},
  ]
    \addplot+[mark=none, myred] table[x=x, y=mean] {\hcdtdfive};
    \addplot[name path=upperb1, draw=none] table[x=x, y expr=\thisrow{mean}+\thisrow{std}] {\hcdtdfive};
    \addplot[name path=lowerb1, draw=none] table[x=x, y expr=\thisrow{mean}-\thisrow{std}] {\hcdtdfive};
    \addplot[fill=myred, fill opacity=0.2] fill between[of=upperb1 and lowerb1];

    \addplot+[mark=none, myblue, dash dot] table[x=x, y=mean] {\hcnaivefive};
    \addplot[name path=upperb1, draw=none] table[x=x, y expr=\thisrow{mean}+\thisrow{std}] {\hcnaivefive};
    \addplot[name path=lowerb1, draw=none] table[x=x, y expr=\thisrow{mean}-\thisrow{std}] {\hcnaivefive};
    \addplot[fill=myblue, fill opacity=0.2] fill between[of=upperb1 and lowerb1];
    
    \addplot+[mark=none, mygray, dashed] table[x=x, y=mean] {\hcbasefive};
    \addplot[name path=upperg1, draw=none] table[x=x, y expr=\thisrow{mean}+\thisrow{std}] {\hcbasefive};
    \addplot[name path=lowerg1, draw=none] table[x=x, y expr=\thisrow{mean}-\thisrow{std}] {\hcbasefive};
    \addplot[fill=mygray, fill opacity=0.2] fill between[of=upperg1 and lowerg1];
    
\nextgroupplot[title={Ant (noise: 0.00)},
    ylabel={\text{Episodic return}},
    ylabel style={
      font=\tiny,
      yshift=-0.2em  
    }, 
    scaled y ticks=base 10:-3,
    ]
    \addplot+[mark=none, myred] table[x=x, y=mean] {\antdtdzero};
    \addplot[name path=upperb1, draw=none] table[x=x, y expr=\thisrow{mean}+\thisrow{std}] {\antdtdzero};
    \addplot[name path=lowerb1, draw=none] table[x=x, y expr=\thisrow{mean}-\thisrow{std}] {\antdtdzero};
    \addplot[fill=myred, fill opacity=0.2] fill between[of=upperb1 and lowerb1];

    \addplot+[mark=none, myblue, dash dot] table[x=x, y=mean] {\antnaiveone};
    \addplot[name path=upperb1, draw=none] table[x=x, y expr=\thisrow{mean}+\thisrow{std}] {\antnaiveone};
    \addplot[name path=lowerb1, draw=none] table[x=x, y expr=\thisrow{mean}-\thisrow{std}] {\antnaiveone};
    \addplot[fill=myblue, fill opacity=0.2] fill between[of=upperb1 and lowerb1];
    
    \addplot+[mark=none, mygray, dashed] table[x=x, y=mean] {\antbaseone};
    \addplot[name path=upperg1, draw=none] table[x=x, y expr=\thisrow{mean}+\thisrow{std}] {\antbaseone};
    \addplot[name path=lowerg1, draw=none] table[x=x, y expr=\thisrow{mean}-\thisrow{std}] {\antbaseone};
    \addplot[fill=mygray, fill opacity=0.2] fill between[of=upperg1 and lowerg1];

\nextgroupplot[title={Ant (noise: 0.01)}, scaled y ticks=base 10:-3,]
    \addplot+[mark=none, myred] table[x=x, y=mean] {\antdtdone};
    \addplot[name path=upperb1, draw=none] table[x=x, y expr=\thisrow{mean}+\thisrow{std}] {\antdtdone};
    \addplot[name path=lowerb1, draw=none] table[x=x, y expr=\thisrow{mean}-\thisrow{std}] {\antdtdone};
    \addplot[fill=myred, fill opacity=0.2] fill between[of=upperb1 and lowerb1];

    \addplot+[mark=none, myblue, dash dot] table[x=x, y=mean] {\antnaiveone};
    \addplot[name path=upperb1, draw=none] table[x=x, y expr=\thisrow{mean}+\thisrow{std}] {\antnaiveone};
    \addplot[name path=lowerb1, draw=none] table[x=x, y expr=\thisrow{mean}-\thisrow{std}] {\antnaiveone};
    \addplot[fill=myblue, fill opacity=0.2] fill between[of=upperb1 and lowerb1];
    
    \addplot+[mark=none, mygray, dashed] table[x=x, y=mean] {\antbaseone};
    \addplot[name path=upperg1, draw=none] table[x=x, y expr=\thisrow{mean}+\thisrow{std}] {\antbaseone};
    \addplot[name path=lowerg1, draw=none] table[x=x, y expr=\thisrow{mean}-\thisrow{std}] {\antbaseone};
    \addplot[fill=mygray, fill opacity=0.2] fill between[of=upperg1 and lowerg1];

\nextgroupplot[title={Ant (noise: 0.05)}, scaled y ticks=base 10:-3,]
    \addplot+[mark=none, myred] table[x=x, y=mean] {\antdtdfive};
    \addplot[name path=upperb1, draw=none] table[x=x, y expr=\thisrow{mean}+\thisrow{std}] {\antdtdfive};
    \addplot[name path=lowerb1, draw=none] table[x=x, y expr=\thisrow{mean}-\thisrow{std}] {\antdtdfive};
    \addplot[fill=myred, fill opacity=0.2] fill between[of=upperb1 and lowerb1];

    \addplot+[mark=none, myblue, dash dot] table[x=x, y=mean] {\antnaivefive};
    \addplot[name path=upperb1, draw=none] table[x=x, y expr=\thisrow{mean}+\thisrow{std}] {\antnaivefive};
    \addplot[name path=lowerb1, draw=none] table[x=x, y expr=\thisrow{mean}-\thisrow{std}] {\antnaivefive};
    \addplot[fill=myblue, fill opacity=0.2] fill between[of=upperb1 and lowerb1];
    
    \addplot+[mark=none, mygray, dashed] table[x=x, y=mean] {\antbasefive};
    \addplot[name path=upperg1, draw=none] table[x=x, y expr=\thisrow{mean}+\thisrow{std}] {\antbasefive};
    \addplot[name path=lowerg1, draw=none] table[x=x, y expr=\thisrow{mean}-\thisrow{std}] {\antbasefive};
    \addplot[fill=mygray, fill opacity=0.2] fill between[of=upperg1 and lowerg1];

\nextgroupplot[title={Humanoid (noise: 0.00)},
  xlabel={\text{Total episode step}},
  xlabel style={
    font=\tiny,
    yshift=0.2em  
  },
  ylabel={\text{Episodic return}},
  ylabel style={
    font=\tiny,
    yshift=-0.2em  
  }, 
  scaled y ticks=base 10:-2,
  ]
    \addplot+[mark=none, myred] table[x=x, y=mean] {\humdtdzero};
    \addplot[name path=upperb1, draw=none] table[x=x, y expr=\thisrow{mean}+\thisrow{std}] {\humdtdzero};
    \addplot[name path=lowerb1, draw=none] table[x=x, y expr=\thisrow{mean}-\thisrow{std}] {\humdtdzero};
    \addplot[fill=myred, fill opacity=0.2] fill between[of=upperb1 and lowerb1];

    \addplot+[mark=none, myblue, dash dot] table[x=x, y=mean] {\humnaivezero};
    \addplot[name path=upperb1, draw=none] table[x=x, y expr=\thisrow{mean}+\thisrow{std}] {\humnaivezero};
    \addplot[name path=lowerb1, draw=none] table[x=x, y expr=\thisrow{mean}-\thisrow{std}] {\humnaivezero};
    \addplot[fill=myblue, fill opacity=0.2] fill between[of=upperb1 and lowerb1];
    
    \addplot+[mark=none, mygray, dashed] table[x=x, y=mean] {\humbasezero};
    \addplot[name path=upperg1, draw=none] table[x=x, y expr=\thisrow{mean}+\thisrow{std}] {\humbasezero};
    \addplot[name path=lowerg1, draw=none] table[x=x, y expr=\thisrow{mean}-\thisrow{std}] {\humbasezero};
    \addplot[fill=mygray, fill opacity=0.2] fill between[of=upperg1 and lowerg1];
    
\nextgroupplot[title={Humanoid (noise: 0.01)},
  xlabel={\text{Total episode step}},
  xlabel style={
    font=\tiny,
    yshift=0.2em  
  },
  scaled y ticks=base 10:-2,
  ]
    \addplot+[mark=none, myred] table[x=x, y=mean] {\humdtdone};
    \addplot[name path=upperb1, draw=none] table[x=x, y expr=\thisrow{mean}+\thisrow{std}] {\humdtdone};
    \addplot[name path=lowerb1, draw=none] table[x=x, y expr=\thisrow{mean}-\thisrow{std}] {\humdtdone};
    \addplot[fill=myred, fill opacity=0.2] fill between[of=upperb1 and lowerb1];

    \addplot+[mark=none, myblue, dash dot] table[x=x, y=mean] {\humnaiveone};
    \addplot[name path=upperb1, draw=none] table[x=x, y expr=\thisrow{mean}+\thisrow{std}] {\humnaiveone};
    \addplot[name path=lowerb1, draw=none] table[x=x, y expr=\thisrow{mean}-\thisrow{std}] {\humnaiveone};
    \addplot[fill=myblue, fill opacity=0.2] fill between[of=upperb1 and lowerb1];
    
    \addplot+[mark=none, mygray, dashed] table[x=x, y=mean] {\humbaseone};
    \addplot[name path=upperg1, draw=none] table[x=x, y expr=\thisrow{mean}+\thisrow{std}] {\humbaseone};
    \addplot[name path=lowerg1, draw=none] table[x=x, y expr=\thisrow{mean}-\thisrow{std}] {\humbaseone};
    \addplot[fill=mygray, fill opacity=0.2] fill between[of=upperg1 and lowerg1];
    
\nextgroupplot[title={Humanoid (noise: 0.05)},
  xlabel={\text{Total episode step}},
  xlabel style={
    font=\tiny,
    yshift=0.2em  
  },
  scaled y ticks=base 10:-2,
  ]
    \addplot+[mark=none, myred] table[x=x, y=mean] {\humdtdfive};
    \addplot[name path=upperb1, draw=none] table[x=x, y expr=\thisrow{mean}+\thisrow{std}] {\humdtdfive};
    \addplot[name path=lowerb1, draw=none] table[x=x, y expr=\thisrow{mean}-\thisrow{std}] {\humdtdfive};
    \addplot[fill=myred, fill opacity=0.2] fill between[of=upperb1 and lowerb1];

    \addplot+[mark=none, myblue, dash dot] table[x=x, y=mean] {\humnaivefive};
    \addplot[name path=upperb1, draw=none] table[x=x, y expr=\thisrow{mean}+\thisrow{std}] {\humnaivefive};
    \addplot[name path=lowerb1, draw=none] table[x=x, y expr=\thisrow{mean}-\thisrow{std}] {\humnaivefive};
    \addplot[fill=myblue, fill opacity=0.2] fill between[of=upperb1 and lowerb1];
    
    \addplot+[mark=none, mygray, dashed] table[x=x, y=mean] {\humbasefive};
    \addplot[name path=upperg1, draw=none] table[x=x, y expr=\thisrow{mean}+\thisrow{std}] {\humbasefive};
    \addplot[name path=lowerg1, draw=none] table[x=x, y expr=\thisrow{mean}-\thisrow{std}] {\humbasefive};
    \addplot[fill=mygray, fill opacity=0.2] fill between[of=upperg1 and lowerg1];

\end{groupplot}
\end{tikzpicture}}

\caption{Performance of \textcolor{mygray}{TD}, \textcolor{myblue}{$\beta$-naive-dTD}, and \textcolor{myred}{$\beta$-dTD} on continuous control benchmark. Each column corresponds to different noise levels ($\text{coef} = 0.00, 0.01, 0.05$), and each row corresponds to different environments. The tuned $\beta$ values for $\beta$-naive-dTD were {0.08, 0.07, 0.23, 0.02} and for $\beta$-dTD were {0.57, 0.74, 0.24, 0.33} in Hopper, HalfCheetah, Ant, and Humanoid, respectively.}
\label{fig:return_curve_td_dtd}
\end{figure}
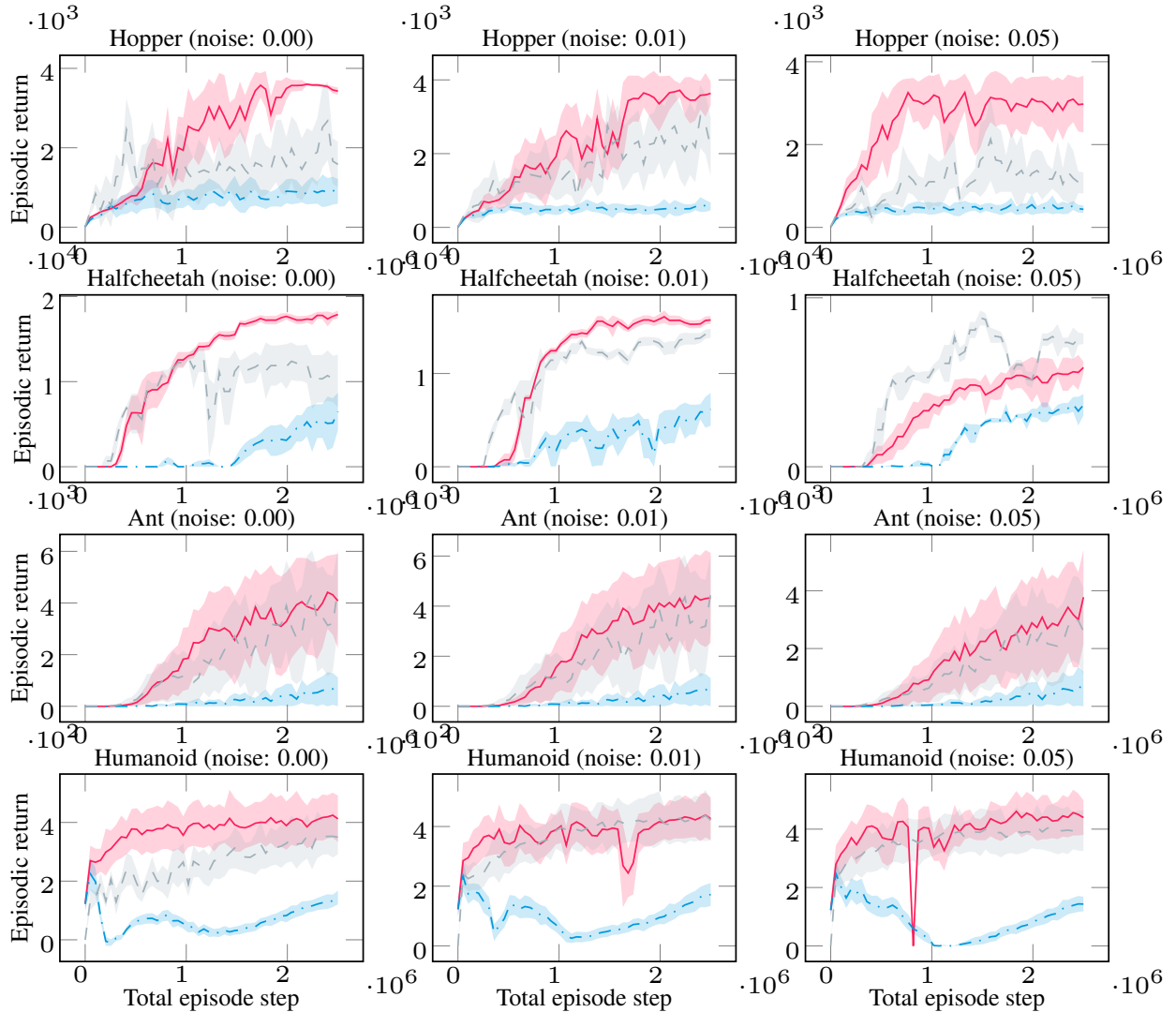

\subsection{Results and discussion}

\paragraph{Comparative evaluation}
We compare (the variants of) the proposed method and the baseline:
\begin{description}[topsep=0pt,leftmargin=1em]
\item[($\beta$-naive-dTD vs. $\beta$-dTD)]
In Figure~\ref{fig:return_curve_td_dtd}, we can observe that $\beta$-dTD consistently outperforms $\beta$-naive-dTD. In all the environments, the optimized values of $\beta$ for $\beta$-naive-dTD were quite small, suggesting that the effective update rule of $\beta$-naive-dTD became close to that of the standard TD. Despite such a fact, however, the performance of $\beta$-naive-dTD remains significantly worse than the standard TD. There are two possible explanations: (1) the $\beta$ value chosen for $\beta$-naive-dTD was actually still not small enough to fully eliminate the adverse effect of the naive-dTD term; and (2) the TD-related parameters in $\beta$-naive-dTD were only suboptimally tuned because hyperparameter tuning resources were allocated mainly to optimizing $\beta$. These factors may jointly account for the unexpectedly poor performance of $\beta$-naive-dTD.
\item[(TD vs. $\beta$-dTD)]
As shown in Figure~\ref{fig:return_curve_td_dtd}, $\beta$-dTD outperforms TD or achieves comparable performance in all cases. While the degree of improvement varies, the final performance of TD and $\beta$-dTD tends to converge, which is not very surprising because both dTD and TD are derived from the same Bellman equation, and the resulting value functions should thus be similar to each other eventually. Nevertheless, dTD has the advantage of implicitly utilizing continuity information during training, which enables it to make more informative updates. Consequently, although the final performance may be comparable, $\beta$-dTD tends to show a faster rate of improvement relative to TD. 
\end{description}

\paragraph{Significance of dTD}
In contrast to $\beta$-naive-dTD, the weight $\beta$ in $\beta$-dTD is not exceedingly small. Notably, in the Halfcheetah environment, $\beta$ assumes a relatively large value of 0.74. This indicates that dTD retains a meaningful impact on the learning process.

\paragraph{Impact of process noise}
In terms of robustness to process noise, both $\beta$-dTD and TD exhibit similar performance. When the noise level is $\text{coef}=0.01$, neither method experiences significant degradation in performance. However, when the noise level is increased to $\text{coef}=0.05$, both $\beta$-dTD and TD show similar reduction in performance, particularly in environments like Ant and Halfcheetah.

\section{Conclusion}
We have presented differential TD (dTD), a temporal difference method based on the HJB equation. In contrast to approaches based on transition kernels, the proposed method can incorporate the continuity of dynamics into the learning process without knowing the dynamics. We have shown empirical results for a variety of continuous control environments with different time intervals.
The empirical results highlight the potential advantages of dTD in terms of learning speed and efficiency while also implying that stability concerns may exist in practice, which led to the introduction of the robust $\beta$-dTD update. Although the current paper focuses on the theoretical development of dTD, these observations are useful and also warrant further empirical exploration.

We have also analyzed the conditions under which the continuous-time dynamics of the HJB equation exhibit exponential stability toward the unique fixed point. This stability property, proven using techniques from linear elliptic PDE theory, is crucial for showing the theoretical convergence of the idealized iterative scheme. However, a drawback is that the sufficient conditions we identify, such as the requirement for a bounded domain and uniform ellipticity (Assumption~\ref{as:hjb_assumptions}), are often hard to maintain or verify in the context of deep RL with function approximation.

Future work includes bridging the gap between the theoretical exponential stability and the practical stability of dTD updates (e.g., by ensuring the coercivity condition in practice), reducing the variance of learning by improved estimators or regularization, and extending the wide range of existing TD-based techniques to the dTD framework.

\section*{Acknowledgements}
NT was supported by JST PRESTO JPMJPR24T6, JSPS JP20K19869, and JSPS JP25H01454.


\newpage
\appendix
\onecolumn
\section{Mathematical Details}
\subsection{Justification for the Continuous RL Formulation}
\label{app:justification}
In Section~3, we modeled the evolution of the state under a stochastic policy $\pi$ by the controlled SDE

\begin{equation}
dS_t = \mu(S_t, A_t)\, dt + \sigma(S_t, A_t)\, dB_t, \quad A_t \sim \pi(\cdot | S_t).\notag
\end{equation}

Here, the control is applied in the form of action samples drawn from a stochastic policy at each time step. While this formulation closely reflects the sampling-based behavior in RL, it raises a technical challenge: the presence of external randomness in addition to the intrinsic Brownian noise introduces analytical difficulties. As a result, the well-definedness of this SDE is not immediately obvious.

To address this issue, many prior works (e.g.,~\cite{wzz:20, jz:22, jz2:22, jz:23, zty:23}) adopt the averaged dynamics, denoted by $(\widetilde{S}_t)_{t \geq 0}$, whose distribution at each time $t$ is known to coincide with that of the original one under the same initial condition~\citep{wzz:20}. Specifically, the averaged dynamics is defined as

\begin{equation}
d\widetilde{S}_t = \widetilde{\mu}(\widetilde{S}_t, \pi)dt + \widetilde{\sigma}(\widetilde{S}_t, \pi)d\widetilde{B}_t,\notag
\end{equation}

where $\widetilde{\mu}(s, \pi)=\int_{\mathcal A}\mu(s, a)\pi(a)da,\ \widetilde{\sigma}(s, a) = \left({\int_{\mathcal A} \sigma(s, a)\sigma^\top(s, a)\pi(a)da}\right)^{\frac{1}{2}}$ and $(\widetilde{B}_t)_{t\geq 0}$ is the $m$-dimensional Brownian motion. Since the averaged dynamics no longer involves the external randomness induced by stochastic action selection, its well-definedness is ensured by classical SDE theory under standard assumptions such as Lipschitz continuity and a linear growth condition.

Since the marginal distributions of the two dynamics coincide, the corresponding value functions also coincide:
\begin{align}
V^\pi(s) &= \mathbb{E}_{p_\pi} \left[ \int_{t}^{\infty} \text{e}^{-\gamma(\tau - t)} \rho(S_\tau, A_\tau)\, d\tau \,\middle|\, S_t = s \right] \notag \\
&= \mathbb{E}_{\widetilde{p}} \left[ \int_{t}^{\infty} \text{e}^{-\gamma(\tau - t)} \widetilde{\rho}(\widetilde{S}_\tau, \pi)\, d\tau \,\middle|\, \widetilde{S}_t = s \right] \notag \\
&=: \widetilde{V}^\pi(s) \notag,
\end{align}
where $\widetilde{\rho}(s, \pi) := \int_{\mathcal A} \rho(s, a)\, \pi(a)\, da$. Hence the value function above is itself well defined and raises no analytical issues.

\subsection{Ito formula}
\label{app:ito}

\label{app:hjb2}
The Bellman equation is given by:
\begin{align}
V^{*}(s_t) = \underset{\pi}{\max}\ \mathbb E_{p_{\pi}}\left[\rho(s_t, A_t)\Delta t + \text{e}^{-\gamma \Delta t} V^{*}(S_{t+\Delta t})\right]. \notag
\end{align}

Assuming that a stochastic process $(S_t)_{t\geq 0}$ follows the SDE (1), the term $V^*(S_{t+\Delta t})$ can be further expanded using It\^{o}'s lemma:
\begin{align}
V^{*}(s_t) = \max_{\pi}\ 
\mathbb{E}_{p_{\pi}}\Big[\, &\rho(s_t, A_t)\Delta t 
+ \mathrm{e}^{-\gamma \Delta t} V^{*}(S_{t+\Delta t}) \,\Big] \notag \\
= \max_{\pi}\ 
\mathbb{E}_{p_{\pi}}\Bigg[\, &\rho(s_t, A_t)\Delta t 
+ \mathrm{e}^{-\gamma \Delta t} \Bigg\{V^{*}(s_t) 
+ \bigg(\sum_{i=1}^{n} \mu^i(s_t, A_t) 
\frac{\partial V^{*}(s)}{\partial s_i} \bigg|_{s_t} \notag \\
&+ \frac{1}{2} \sum_{i=1}^{n} \sum_{j=1}^{n} 
[\sigma(s_t, A_t)\sigma^\top(s_t, A_t)]^{ij} 
\frac{\partial^2 V^{*}(s)}{\partial s_i \partial s_j} \bigg|_{s_t}\bigg)\Delta t \notag \\
&+ \sum_{i=1}^{n} \sum_{j=1}^{m} \sigma_j^i(s_t, A_t) 
\frac{\partial V^{*}(s)}{\partial s_i} \bigg|_{s_t} 
(B_{t+\Delta t}^{j} - B_t^j) 
+ O((\Delta t)^{3/2}) \Bigg\} \Bigg] \notag \\
= \max_{\pi}\ 
\mathbb{E}_{p_{\pi}}\Bigg[ \, &\rho(s_t, A_t)\Delta t 
+ \mathrm{e}^{-\gamma \Delta t} \Bigg\{ V^{*}(s_t) 
+ \Bigg(\sum_{i=1}^{n} \mu^i(s_t, A_t) 
\frac{\partial V^{*}(s)}{\partial s_i}\bigg|_{s_t} \notag \\
&+ \frac{1}{2} \sum_{i=1}^{n} \sum_{j=1}^{n} 
[\sigma(s_t, A_t)\sigma^\top(s_t, A_t)]^{ij} 
\frac{\partial^2 V^{*}(s)}{\partial s_i \partial s_j}\bigg|_{s_t} 
\Bigg) \Delta t 
+ O((\Delta t)^{3/2}) \Bigg\} \Bigg]. \notag
\end{align}

Simplifying the equation and taking the limit as $\Delta t\rightarrow0$, we have the condition for the optimal value function:
\begin{align}
V^{*}(s_t) = \frac{1}{\gamma} \max_{\pi}\ 
\mathbb{E}_{p_{\pi}} \Bigg[
    \rho(s_t, A_t) 
    &+ \sum_{i=1}^{n} \mu^i(s_t, A_t) \frac{\partial V^{*}(s)}{\partial s_i} \bigg|_{s_t} \notag \\
    &+ \frac{1}{2} \sum_{i=1}^{n} \sum_{j=1}^{n} 
        [\sigma(s_t, A_t)\sigma^\top(s_t, A_t)]^{ij}
        \frac{\partial^2 V^{*}(s)}{\partial s_i \partial s_j} \bigg|_{s_t} 
\Bigg]. \notag
\end{align}

\subsection{Convergence of dTD}
\label{app:convergence}
\subsubsection{Proof of Lemma \ref{lem:existence_uniqueness}}
\label{app:proof_of_lemma1}
\begin{proof}
We consider the PDE $(\gamma I-\mathcal L^\pi)V=\bar\rho$ with
\begin{equation*}
(\mathcal L^\pi V)(s)=\sum_{i=1}^n \bar\mu^i(s)\,\frac{\partial V}{\partial s^i}
+\frac12\sum_{i,j=1}^n D^{ij}(s)\,\frac{\partial^2 V}{\partial s^i\partial s^j},
\quad
D(s)=\mathbb E_{a\sim\pi(\cdot|s)}[\sigma(s,a)\sigma^\top(s,a)].
\end{equation*}
Multiply by $v\in H^1(S)$, integrate over $S$, and integrate by parts once in $s^i$ for the second–order term. Using the homogeneous Neumann boundary condition $n\cdot(D\nabla V)=0$ on $\partial S$, we obtain
\begin{equation*}
\int_S(\gamma V-\mathcal L^\pi V)\,v
= \int_S \gamma Vv
+\tfrac12\int_S \sum_{i,j} D^{ij}\,\partial_j V\,\partial_i v
-\int_S \sum_i \bar\mu^i\,\partial_i V\,v
-\tfrac12\int_S \sum_{i,j} (\partial_i D^{ij})\,\partial_j V\,v .
\end{equation*}
Define
\begin{equation*}
B(V,v):=\int_S \gamma Vv
+\tfrac12\int_S \sum_{i,j} D^{ij}\,\partial_j V\,\partial_i v
-\int_S \sum_i \bar\mu^i\,\partial_i V\,v
-\tfrac12\int_S \sum_{i,j} (\partial_i D^{ij})\,\partial_j V\,v,
\quad
f(v):=\int_S \bar\rho\,v .
\end{equation*}

\emph{Boundedness.}
By the bounded/Lipschitz coefficients in Assumption~\ref{as:hjb_assumptions}.3, there exists $C>0$ such that
\begin{equation*}
|B(V,v)|\le C\,\|V\|_{H^1(S)}\,\|v\|_{H^1(S)},\qquad
|f(v)|\le \|\bar\rho\|_{L^2(S)}\,\|v\|_{L^2(S)}.
\end{equation*}

\emph{Coercivity.}
Using uniform ellipticity (Assumption~\ref{as:hjb_assumptions}.4), we have
\begin{equation*}
B(V,V)
= \int_S \gamma V^2
+\tfrac12\int_S (\nabla V)^\top D\,\nabla V
-\int_S \bar\mu\cdot(\nabla V)\,V
-\tfrac12\int_S (\mathrm{div}\,D)\cdot(\nabla V)\,V .
\end{equation*}
Hence
\begin{equation*}
B(V,V)\ge \gamma\|V\|_{L^2}^2+\frac{\alpha}{2}\|\nabla V\|_{L^2}^2
-\Big(\|\bar\mu\|_\infty+\tfrac12\|\mathrm{div}\,D\|_\infty\Big)\|\nabla V\|_{L^2}\|V\|_{L^2}.
\end{equation*}
By Young's inequality,
\begin{equation*}
\Big(\|\bar\mu\|_\infty+\tfrac12\|\mathrm{div}\,D\|_\infty\Big)\|\nabla V\|_{L^2}\|V\|_{L^2}
\le \tfrac{\alpha}{4}\|\nabla V\|_{L^2}^2
+\frac{(\|\bar\mu\|_\infty+\tfrac12\|\mathrm{div}\,D\|_\infty)^2}{\alpha}\|V\|_{L^2}^2.
\end{equation*}
Therefore
\begin{equation*}
B(V,V)\ge
\Big(\gamma-\frac{(\|\bar\mu\|_\infty+\tfrac12\|\mathrm{div}\,D\|_\infty)^2}{\alpha}\Big)\|V\|_{L^2}^2
+\frac{\alpha}{4}\|\nabla V\|_{L^2}^2
\ge c\,\|V\|_{H^1(S)}^2
\end{equation*}
for some $c>0$ ensured by Assumption~\ref{as:hjb_assumptions}.5.

Since $H^1(S)$ is a Hilbert space, $B$ is bounded and coercive, and $f$ is bounded, the Lax--Milgram Theorem implies there exists a unique weak solution $V^\pi\in H^1(S)$.
\end{proof}

\subsubsection{Proof of Proposition \ref{prop:convergence}}

\begin{proof}
Set
\begin{equation*}
J(t):=\dfrac{1}{2}\|V(t)-V^\pi\|_{L^2(S)}^2.
\end{equation*}
Since $TV=\frac{1}{\gamma}\big(\bar\rho+\mathcal L^\pi V\big)$ is affine and $TV^\pi=V^\pi$, we have
\begin{equation*}
\frac{\partial}{\partial t}\big(V(t)-V^\pi\big)=T V(t)-V(t)-\big(TV^\pi-V^\pi\big)
=-\gamma^{-1}\,(\gamma I-\mathcal L^\pi)\big(V(t)-V^\pi\big).
\end{equation*}
Differentiating $J$ and using the bilinear form $B(\cdot,\cdot)$ from yields
\begin{align*}
\frac{dJ}{dt}
&=\Big\langle V(t)-V^\pi,\ \frac{\partial}{\partial t}\big(V(t)-V^\pi\big)\Big\rangle \\
&=-\frac{1}{\gamma}\,\Big\langle V(t)-V^\pi,\ (\gamma I-\mathcal L^\pi)\big(V(t)-V^\pi\big)\Big\rangle \\
&=-\frac{1}{\gamma}\,B\!\left(V(t)-V^\pi,\,V(t)-V^\pi\right).
\end{align*}
By coercivity of $B$, there exists $c>0$ such that
\begin{equation*}
B\!\left(V(t)-V^\pi,\,V(t)-V^\pi\right)\ \ge\ c\,\|V(t)-V^\pi\|_{H^1(S)}^2\ \ge\ 2c\,J(t).
\end{equation*}
Therefore,
\begin{equation*}
\frac{dJ}{dt}\ \le\ -\frac{2c}{\gamma}\,J(t).
\end{equation*}
By Grönwall's inequality,
\begin{equation*}
J(t)\ \le\ J(0)\,e^{-(2c/\gamma)t},
\end{equation*}
which implies the claimed exponential convergence in $L^2(S)$:
\begin{equation*}
\|V(t)-V^\pi\|_{L^2(S)}\ \le e^{-(c/\gamma)t} \|V(0)-V^\pi\|_{L^2(S)}.
\end{equation*}
\end{proof}


\section{Implementation Details}\label{app:implementation}

\subsection{Algorithm}
\label{app:algorithms}

The procedures for policy evaluation are summarized in Algorithm~\ref{alg:pe_dtd}.

\begin{algorithm}[H]
   \caption{Policy evaluation with dTD}
   \label{alg:pe_dtd}
\begin{algorithmic}
   \REQUIRE policy $\pi$
   \ENSURE $V_\theta$
   \STATE Initialize value function $V_{\theta}$ with random parameter $\theta$
   \FOR{$\text{each training step}$}
    \STATE Initialize buffer \(\mathcal{D} = \emptyset\) and initial state $s_0$
      \FOR{$\text{each environment step}$} 
        \STATE $a_t\sim\pi(\cdot|s_t)$
        \STATE $s_{t+\Delta t}\sim p(\cdot |s_t, a_t)$
        \STATE $\mathcal D\leftarrow\mathcal D\cup (s_t, a_t, s_{t+\Delta t}, \rho_t)$
      \ENDFOR
      \FOR{$\text{each update step}$} 
        \STATE Sample a batch of $D$ random transitions from $\mathcal D$
        \STATE $\bar\theta\leftarrow\theta$
        \STATE $y_d \leftarrow -\rho_t^d - \gamma V_{\bar\theta}(s_{t+\Delta t}^d)$
        \STATE $\text{pred}_d \leftarrow \displaystyle\sum_{i=1}^{n}\frac{(s_{t+\Delta t}^{d,i} - s_t^{d,i})}{\Delta t} \frac{\partial V_{\theta}(s)}{\partial s_i}\bigg|_{s_{t}^{d}} + \displaystyle\frac{1}{2} \sum_{i=1}^{n} \sum_{j=1}^{n} \frac{(s_{t+\Delta t}^{d,i} - s_t^{d,i})(s_{t+\Delta t}^{d,j} - s_t^{d,j})}{\Delta t} \frac{\partial^2 V_{\theta}(s)}{\partial s_i \partial s_j}\bigg|_{s_{t}^{d}}$
        \STATE Update parameter $\theta$ using gradient descent method
        \STATE $\theta \leftarrow \text{argmin}_{\theta}\frac{1}{D}\sum_{d=1}^D(y_d-\text{pred}_d)^2$
      \ENDFOR
   \ENDFOR
\end{algorithmic}
\end{algorithm}

\subsection{Hyperparameters} 

The search space of the hyperparameters is summarized in Table~\ref{table:hpo}. The values chosen finally are summarized in Tables~\ref{table:hpo_result}.

\begin{table*}[ht]
\centering
\caption{Hyperparameter search space}
\begin{tabular}{ll}
\toprule
Hyperparameter               & Search Space                        \\
\midrule
\shortstack{environment steps per update \\ (number of parallel environment: 64)}& \shortstack{$\{8, 16, 32\}$ \\ {}}                    \\
number of epochs per update  & range(5, 20)                        \\
minibatch size               & $\{256, 512\}$                       \\
learning rate                & $\log(\text{interval}(1e-6, 5e-3))$ \\
normalize advantage          & $\{\text{True}, \text{False}\}$                  \\
gae lambda                   & interval(0.8, 0.9999)               \\
clip range                   & interval(0.0, 0.9)                  \\
entropy coefficient          & interval(0.0, 0.3)                  \\
value loss weight            & interval(0.0, 1.0)                  \\
mixture raio $\beta$                      & interval(0.0, 1.0)         \\
\bottomrule
\end{tabular}
\label{table:hpo}
\end{table*}

\begin{table*}[ht]
\centering
\caption{Best hyperparameters for PPO with TD and $\beta$-dTD across environments}
\resizebox{\textwidth}{!}{
\begin{tabular}{l|cccc|cccc}
\toprule
\multirow{2}{*}{Hyperparameter} & \multicolumn{4}{c|}{TD} & \multicolumn{4}{c}{$\beta$-dTD} \\
& Hopper & Halfcheetah & Ant & Humanoid & Hopper & Halfcheetah & Ant & Humanoid \\
\midrule
environment steps/update       & 32    & 16    & 8     & 16    & 32    & 8     & 32    & 16    \\
epochs/update                  & 7     & 5     & 11    & 15    & 19    & 9     & 16    & 10    \\
minibatch size                 & 512   & 256   & 512   & 512   & 256   & 256   & 256   & 256   \\
learning rate                  & 1.18e-3 & 5.93e-4 & 3.71e-4 & 1.40e-3 & 3.52e-4 & 3.29e-4 & 7.94e-5 & 1.08e-3 \\
normalize advantage            & False & False & False & False & False & False & False & False \\
GAE lambda                     & 0.886 & 0.833 & 0.935 & 0.999 & 0.998 & 0.908 & 0.805 & 0.888 \\
clip range                     & 0.439 & 0.268 & 0.425 & 0.063 & 0.075 & 0.040 & 0.520 & 0.713 \\
entropy coefficient            & 0.121 & 0.018 & 0.162 & 0.021 & 0.046 & 0.011 & 0.133 & 0.002 \\
value loss weight              & 0.049 & 0.513 & 0.711 & 0.091 & 0.675 & 0.268 & 0.274 & 0.054 \\
mixture ratio $\beta$         & ---   & ---   & ---   & ---   & 0.572 & 0.742 & 0.241 & 0.332 \\
\bottomrule
\end{tabular}
}
\label{table:hpo_result}
\end{table*}

\subsection{Computing Infrastructure and Reproducibility}
\label{app:infra}

\paragraph{Computing infrastructure}
Experiments were conducted on a machine with four NVIDIA Tesla V100 GPUs (32GB each) and an Intel Xeon E5-2698 v4 CPU. Although all experiments can be executed on a single GPU, multiple GPUs were used to run independent trials in parallel for efficiency.

\paragraph{Training time}
Hyperparameter tuning typically took 6--9 hours depending on the environment. Training time for the final runs depended on the environment and ranged from 10 to 60 minutes.

\paragraph{Reproducibility}
All experiments were conducted with the random seed fixed in the training scripts. However, MuJoCo (accessed via Brax) uses its own internal random seed that is not directly controllable, so full determinism cannot be ensured. The code is available at \url{https://github.com/4thhia/differential_TD} for reproducibility.

\end{document}